\documentclass[letterpaper, 10 pt, conference]{ieeeconf}  

\IEEEoverridecommandlockouts                              

\usepackage{graphics} 
\usepackage{booktabs}
\usepackage{epsfig} 
\usepackage{times} 
\usepackage{amsmath} 
\usepackage{amssymb}  
\usepackage{mathtools}
\usepackage{tikz}
\usepackage{standalone}
\usepackage{bondgraphs}
\usepackage{url}
\usepackage{blox}
\usepackage{gensymb}
\usepackage{tabularx}
\usepackage{amsfonts,dcolumn}
\usepackage{blindtext}
\usepackage{rotating}
\usepackage{color}
\usepackage{transparent}
\usepackage{algorithm}
\usepackage{float}
\usepackage{algorithmic}
\usepackage{upgreek}
\usepackage[miama]{emf}
\DeclareMathAlphabet{\mathcal}{OMS}{cmsy}{m}{n}

\usetikzlibrary{shapes,arrows}
\usetikzlibrary{calc, patterns, decorations.pathmorphing, decorations.markings, shapes.geometric, arrows, positioning, angles, automata, positioning}
\usepackage{cleveref}

\usepackage[integrals]{wasysym}

\usepackage[keeplastbox]{flushend}
\usepackage{balance}

\usepackage{textcomp}
\usepackage{lipsum}
\usepackage{cuted}
\allowdisplaybreaks
\usepackage{accents}

\newcommand{\R}{\mathbb{R}}

\DeclarePairedDelimiter{\norm}{\lVert}{\rVert}
\DeclarePairedDelimiter{\abs}{\lvert}{\rvert}%

\newtheorem{proposition}{\bf Proposition}

\mathchardef\mhyphen="2D
\mathchardef\mslash ="202F

\DeclareSymbolFont{bfletters}{OT1}{cmr}{bx}{n}
\DeclareSymbolFontAlphabet{\mathbf}{bfletters}

\DeclareMathOperator*{\argmax}{arg\,max}

\DeclareMathSymbol{A}{\mathalpha}{bfletters}{`A}
\DeclareMathSymbol{B}{\mathalpha}{bfletters}{`B}
\DeclareMathSymbol{C}{\mathalpha}{bfletters}{`C}
\DeclareMathSymbol{D}{\mathalpha}{bfletters}{`D}
\DeclareMathSymbol{E}{\mathalpha}{bfletters}{`E}
\DeclareMathSymbol{F}{\mathalpha}{bfletters}{`F}
\DeclareMathSymbol{G}{\mathalpha}{bfletters}{`G}
\DeclareMathSymbol{H}{\mathalpha}{bfletters}{`H}
\DeclareMathSymbol{I}{\mathalpha}{bfletters}{`I}
\DeclareMathSymbol{J}{\mathalpha}{bfletters}{`J}
\DeclareMathSymbol{K}{\mathalpha}{bfletters}{`K}
\DeclareMathSymbol{L}{\mathalpha}{bfletters}{`L}
\DeclareMathSymbol{M}{\mathalpha}{bfletters}{`M}
\DeclareMathSymbol{N}{\mathalpha}{bfletters}{`N}
\DeclareMathSymbol{O}{\mathalpha}{bfletters}{`O}
\DeclareMathSymbol{P}{\mathalpha}{bfletters}{`P}
\DeclareMathSymbol{Q}{\mathalpha}{bfletters}{`Q}
\DeclareMathSymbol{R}{\mathalpha}{bfletters}{`R}
\DeclareMathSymbol{S}{\mathalpha}{bfletters}{`S}
\DeclareMathSymbol{T}{\mathalpha}{bfletters}{`T}
\DeclareMathSymbol{U}{\mathalpha}{bfletters}{`U}
\DeclareMathSymbol{V}{\mathalpha}{bfletters}{`V}
\DeclareMathSymbol{W}{\mathalpha}{bfletters}{`W}
\DeclareMathSymbol{X}{\mathalpha}{bfletters}{`X}
\DeclareMathSymbol{Y}{\mathalpha}{bfletters}{`Y}
\DeclareMathSymbol{Z}{\mathalpha}{bfletters}{`Z}

\DeclareMathSymbol{a}{\mathalpha}{bfletters}{`a}
\DeclareMathSymbol{b}{\mathalpha}{bfletters}{`b}
\DeclareMathSymbol{c}{\mathalpha}{bfletters}{`c}
\DeclareMathSymbol{d}{\mathalpha}{bfletters}{`d}
\DeclareMathSymbol{e}{\mathalpha}{bfletters}{`e}
\DeclareMathSymbol{f}{\mathalpha}{bfletters}{`f}
\DeclareMathSymbol{g}{\mathalpha}{bfletters}{`g}
\DeclareMathSymbol{h}{\mathalpha}{bfletters}{`h}
\DeclareMathSymbol{i}{\mathalpha}{bfletters}{`i}
\DeclareMathSymbol{j}{\mathalpha}{bfletters}{`j}
\DeclareMathSymbol{k}{\mathalpha}{bfletters}{`k}
\DeclareMathSymbol{l}{\mathalpha}{bfletters}{`l}
\DeclareMathSymbol{m}{\mathalpha}{bfletters}{`m}
\DeclareMathSymbol{n}{\mathalpha}{bfletters}{`n}
\DeclareMathSymbol{o}{\mathalpha}{bfletters}{`o}
\DeclareMathSymbol{p}{\mathalpha}{bfletters}{`p}
\DeclareMathSymbol{q}{\mathalpha}{bfletters}{`q}
\DeclareMathSymbol{r}{\mathalpha}{bfletters}{`r}
\DeclareMathSymbol{s}{\mathalpha}{bfletters}{`s}
\DeclareMathSymbol{t}{\mathalpha}{bfletters}{`t}
\DeclareMathSymbol{u}{\mathalpha}{bfletters}{`u}
\DeclareMathSymbol{v}{\mathalpha}{bfletters}{`v}
\DeclareMathSymbol{w}{\mathalpha}{bfletters}{`w}
\DeclareMathSymbol{x}{\mathalpha}{bfletters}{`x}
\DeclareMathSymbol{y}{\mathalpha}{bfletters}{`y}
\DeclareMathSymbol{z}{\mathalpha}{bfletters}{`z}

\usepackage{graphicx} 
\usepackage{amsmath} 
\usepackage{amssymb}  

\usepackage{setspace}
\let\Algorithm\algorithm
\renewcommand\algorithm[1][]{\Algorithm[#1]\setstretch{1.15}}

\title{\LARGE \bf A Barrier Pair Method for Safe Human-Robot Shared Autonomy}
\author{Binghan He$^{1}$, Mahsa Ghasemi, Ufuk Topcu and Luis Sentis
\thanks{
This work was supported by the National Science Foundation \texttt{\small [grant number 1652113]} and \texttt{\small [grant number 1836900]}.
The authors are with the Department of Mechanical Engineering \texttt{\small(B.H.)}, the Department of Electrical and Computer Engineering \texttt{\small(M.G.)} and the Department of Aerospace Engineering and Engineering Mechanics \texttt{\small(U.T., L.S.)}, The University of Texas at Austin, Austin, TX.
Send correspondence to $^{1}$$\;${\tt\small binghan at utexas dot edu}.
}
}

\newcommand\copyrighttext{%
  \scriptsize 
  Accepted for publication in IEEE Conference on Decision and Control (CDC)
  \textcopyright 2021 IEEE. Personal use of this material is permitted. Permission from IEEE must be obtained for all other uses, in any current or future media, including reprinting/republishing this material for advertising or promotional purposes, creating new collective works, for resale or redistribution to servers or lists, or reuse of any copyrighted component of this work in other works.
  }
\newcommand\copyrightnotice{%
\begin{tikzpicture}[remember picture,overlay]
\node[anchor=south,yshift=10pt] at (current page.south)
{\fbox{\parbox{\dimexpr\textwidth-\fboxsep-\fboxrule\relax}{\copyrighttext}}};
\end{tikzpicture}%
}

\begin{document}

\maketitle
\thispagestyle{empty}
\pagestyle{empty}
\copyrightnotice
\vspace{-10pt}

\begin{abstract}
Shared autonomy provides a framework where a human and an automated system, such as a robot, jointly control the system's behavior, enabling an effective solution for various applications, including human-robot interaction.
However, a challenging problem in shared autonomy is safety because the human input may be unknown and unpredictable, which affects the robot's safety constraints. 
If the human input is a force applied through physical contact with the robot, it also alters the robot’s behavior to maintain safety. 
We address the safety issue of shared autonomy in real-time applications by proposing a two-layer control framework.
In the first layer, we use the history of human input measurements to infer what the human wants the robot to do and define the robot's safety constraints according to that inference.
In the second layer, we formulate a rapidly-exploring random tree of barrier pairs, with each barrier pair composed of a barrier function and a controller. 
Using the controllers in these barrier pairs, the robot is able to maintain its safe operation under the intervention from the human input.
This proposed control framework allows the robot to assist the human while preventing them from encountering safety issues.
We demonstrate the proposed control framework on a simulation of a two-linkage manipulator robot.
\end{abstract}

\section{Introduction}

Unlike full robot autonomy, shared autonomy allows a robot to leverage the perceptual and decision making capabilities of operators while helping them to work more efficiently and accurately \cite{colgate2008safety}. 
Across different fields, such as brain-computer interfaces \cite{carlson2013brain}, autonomous driving \cite{fridman2018human}, and teleoperation \cite{javdani2018shared}, shared autonomy helps us to improve our productivity without completely removing the human from the task at hand.
However, safety becomes critical with shared autonomy, especially when operators and robots interact through physical contact.
On the one hand, the human's objective is not directly measurable but can be inferred based on the robot's sensing of human inputs such as contact forces.
The robot needs this inference of the human's objective to figure out how to assist the human and prevent them from potential accidents.
On the other hand, human inputs can alter the robot’s current path resulting in additional safety concerns.
Therefore in a shared autonomy task, the robot faces a conflict between inferring the human's objectives and maintaining safety under the interaction with human inputs.

For a nonlinear dynamical system such as a robot, safety is usually verified through barrier functions \cite{prajna2004safety}.
Just like Lyapunov functions for stability verification, barrier functions provide sufficient conditions for safety verification.
But barrier functions relax the global convergence requirement of Lyapunov functions and only need to be decreasing at the safety bounds.
Various methods create barrier functions with controllers to enforce safety constraint satisfaction.
For state-space constraints, controllers can be synthesized simultaneously with barrier functions using back-stepping \cite{tee2009barrier} or quadratic programming methods \cite{ames2016control, nguyen2016optimal, nguyen2016exponential}.
Input constraints can also be enforced using semi-definite programming methods \cite{pylorof2016analysis, thomas2018safety}.
By incorporating sampling-based methods into the synthesis of barrier functions and controllers \cite{he2020bp}, robots can also guarantee safe operation with non-convex state-space constraints. 

While the above methods aim to resolve the safety problem for a robot alone, it is a more challenging problem to guarantee safety for a robot that has physical contact with a person. 
This is because humans represent an uncertain dynamical sub-system when physically interacting with robots. The internal states of the human dynamics are usually immeasurable. 
Robust force-based controllers \cite{buerger2007complementary, he2019modeling} address the uncertain human dynamics and achieve complementary stability for human-robot coupled system.
In addition, they can be implemented as output feedback controllers, which bypass the internal human states. 
However, force-based control strategies only consider the robot as a strict follower of the human's trajectory and hence, relies on the human to obey safety constraints.
In human-robot shared autonomy, the robot needs to enforce safety constraints in relation to the human's objective such that it can prevent them from potential accidents.

In this paper, we address the safety problem during shared autonomy using a two-layer control framework.
In the first layer, we define the robot's safety constraints based on the inference of the human's objective. 
In order to understand the human's objective, we propose an intent inference method that integrates the history of human's input, obtained through the sensor readings, to generate a probability distribution over a set of candidate objectives. In particular, the proposed method employs a Boltzmann model of rationality~\cite{baker2007goal,morgenstern1953theory} to characterize the probability of receiving a particular sensor reading at the robot's end-effector based on the human's objective. 
Using this Boltzmann model, the robot keeps track of a probability distribution, called belief, over the candidate objectives and updates it using a Bayesian method. 
Then, the belief is handed to a safety controller in the second layer so that the robot can safely move toward the most probable human objective.

In the second layer, we use the barrier pair rapidly-exploring random tree method \cite{he2020bp} to generate sequences of barrier pairs given different human objectives. 
Each barrier pair comprises a quadratic barrier function and a state feedback controller. 
We synthesize the state feedback controllers in these barrier pairs using a robust control strategy so that the robot can satisfy the safety constraints for different human objectives and reject the human input interventions.
Based on the human intention inference formulated from the first layer, the robot can execute these barrier pair sequences accordingly and help the human to safely accomplish the objective. 
We demonstrate this two-layer control framework on a simulation of a two-linkage manipulator robot, where a human operator uses a keyboard to control a simulated human force exerted on the end-effector of the manipulator robot.

\section{Preliminaries}

In this section, we first overview the basics of multi-body robot dynamics and barrier pair rapidly-exploring random trees. 
Then, we present the formal problem statement.
$\mathsf{a}_\mathrm{i}$ is defined as a polytopic region in the workspace of a robot.
For convenience, $x_\mathsf{a_\mathrm{i}}$ is defined as the geometric center for the region of $\mathsf{a}_\mathrm{i}$ and $\bar{\mathsf{a}}_\mathrm{i} \triangleq \R ^ {\mathrm{n}} \smallsetminus \mathsf{a}_\mathrm{i}$ is defined as a workspace region excluding the set for $\mathsf{a}_\mathrm{i}$.

\subsection{Multi-Body Robot Dynamics}

The Lagrangian dynamics of an n-DOF robot can be expressed as
\begin{equation} \label{eq:lg}
M(q) \cdot \ddot{q} + C(q, \, \dot{q}) \cdot \dot{q} = u + J ^ \top (q) \cdot w
\end{equation}
where $M(q)$ is the matrix of inertia, $C(q, \, \dot{q})$ is the coefficient matrix of Coriolis and centrifugal effects, $J (q)$ is the matrix of Jacobian, $q \triangleq [ \mathit{q}_1, \, \cdots, \, \mathit{q}_\mathrm{n} ] ^ \top$ is the vector of joint positions with $\dot{q}$ and $\ddot{q}$ defined as its first and second order time derivatives, $u \triangleq [ \mathit{u}_1, \, \cdots, \, \mathit{u}_\mathrm{n} ] ^ \top$ is the vector of joint torques and $w \triangleq [ \mathit{w}_1, \, \cdots, \, \mathit{w}_\mathrm{n} ] ^ \top$ is the vector of external forces exerted by the human.
An n-dimensional workspace position vector $x \triangleq [ \mathit{x}_1, \, \cdots, \, \mathit{x}_\mathrm{n} ] ^ \top$ can be calculated from the joint position vector using
\begin{equation} \label{eq:forward}
x = F (q)
\end{equation}
where $F (\cdot)$ represents the forward kinematics.
By linearizing \eqref{eq:lg} and \eqref{eq:forward} around an equilibrium point $[ q_e ^ \top, \, \vec{0} ^ {\, \top}] ^ \top$, we obtain the state-space form
\begin{align}
\begin{bmatrix}
\dot{\tilde{q}} \\
\ddot{\tilde{q}}
\end{bmatrix}
& = 
\begin{bmatrix}
\mathbf{0} & \mathbf{I} \\
\mathbf{0} & - M ^ {-1} (q_e) \cdot C (q_e, \, \vec{0} )
\end{bmatrix}
\begin{bmatrix}
\tilde{q} \\
\dot{\tilde{q}}
\end{bmatrix}
+
\begin{bmatrix}
\mathbf{0} \\
M ^ {-1} (q_e)
\end{bmatrix}
u \notag  \\
& +
\begin{bmatrix}
\mathbf{0} \\
M ^ {-1} (q_e) \cdot J ^ \top (q_e)
\end{bmatrix}
w \label{eq:ss-lg} \\
\tilde{x} \,
& = 
\begin{bmatrix}
J (q_e) & \mathbf{0}
\end{bmatrix}
\begin{bmatrix}
\tilde{q} \\
\dot{\tilde{q}}
\end{bmatrix} \label{eq:ss-jacobian}
\end{align}
where $\tilde{q} \triangleq q - q_e$ and $\tilde{x} \triangleq x - x_e$ with $x_e = F (q_e)$. The partial derivative of $F (q)$ with respect to $q$ is the Jacobian matrix $J (q)$.

\subsection{Barrier Pair Rapidly-Exploring Random Trees}

{\bf Definition 1} \cite{thomas2018safety}:
A \emph{barrier pair} is a pair consisting of a barrier function and a controller $(B,\ k)$ with the following properties
\begin{itemize}
\item[(a)] $-1<B(\tilde{q}, \, \dot{\tilde{q}})\leq 0, u = k (\tilde{q}, \, \dot{\tilde{q}}) \implies \dot B(\tilde{q}, \, \dot{\tilde{q}}) < 0$,
\vspace{3pt}
\item[(b)] $B(\tilde{q}, \, \dot{\tilde{q}})\leq 0 \implies \mathbf{[ \tilde{q} ^ \top, \, \dot{\tilde{q}} ^ \top] ^ \top} \in \mathsf{Z},\ k (\tilde{q}, \, \dot{\tilde{q}}) \in \mathsf{U}$,
\vspace{1pt}
\end{itemize}
where $[ \tilde{q} ^ \top, \, \dot{\tilde{q}} ^ \top] ^ \top \in \mathsf{Z}$ and $u \in \mathsf{U}$ are the state and input constraints. 

If we define the barrier pair as 
\begin{equation} \label{eq:bp}
B = 
\begin{bmatrix}
\tilde{q} \\
\dot{\tilde{q}}
\end{bmatrix} 
^ \top
\mkern-14mu
Q^{-1} 
\mkern-6mu
\begin{bmatrix}
\tilde{q} \\
\dot{\tilde{q}}
\end{bmatrix}
- 1, \quad k = K \begin{bmatrix}
\tilde{q} \\
\dot{\tilde{q}}
\end{bmatrix}
\end{equation}
where $B$ is a quadratic barrier function with a positive definite matrix $Q$ and $k$ is a full state feedback controller,
the barrier pair synthesis becomes a linear matrix inequality ($\mathsf{LMI}$) optimization problem \cite{thomas2018safety}. 

For convenience, we use $(Q, K)$ to represent a barrier pair $(B, k)$ in the form of \eqref{eq:bp} and define $E (\upvarepsilon) \triangleq \{ [ \tilde{q} ^ \top, \, \dot{\tilde{q}} ^ \top] ^ \top \mid [ \tilde{q} ^ \top, \, \dot{\tilde{q}} ^ \top] Q ^ {-1} [ \tilde{q} ^ \top, \, \dot{\tilde{q}} ^ \top] ^ \top \leq \upvarepsilon ^ 2 \}$ as the sub-level set of $B$ corresponding to a value $\upvarepsilon ^ 2 - 1$. Based on Definition~1, the zero sub-level set $E (1)$ of the barrier function $B$ needs to satisfy all constraints defined by $\mathsf{Z}$ and $\mathsf{U}$.


\begin{algorithm}[t]
\caption{$G \leftarrow \texttt{BPRRT}(\mathsf{a}_{\mathsf{0}}, \mathsf{a}_{\mathsf{f}}, \bar{\mathsf{a}}_{\mathsf{1}}, \cdots, \bar{\mathsf{a}}_{\mathrm{n_o}}, \mathsf{Z_0}, \mathsf{U}, \upvarepsilon)$} \label{code:BP-RRT-1}
\begin{algorithmic} [1]
\REQUIRE Initial region $\mathsf{a}_{\mathsf{0}}$, goal region $\mathsf{a}_{\mathsf{f}}$, constraints associated with undesirable regions $\bar{\mathsf{a}}_{\mathsf{1}}, \, \cdots, \, \bar{\mathsf{a}}_{\mathrm{n_o}}$, state space constraint $\mathsf{Z_0}$, input constraint $\mathsf{U}$, scalar $\upvarepsilon$ $(0 < \upvarepsilon \leq 1)$
\ENSURE $\mathsf{BP}$-$\mathsf{RRT}$ graph $G$
\STATE $(Q_{\mathsf{f}}, \, K_{\mathsf{f}}) \leftarrow \texttt{BP} (x_{\mathsf{a}_\mathsf{f}}, \, \mathsf{a}_{\mathsf{f}}, \, \bar{\mathsf{a}}_{\mathsf{1}}, \, \cdots, \, \bar{\mathsf{a}}_{\mathrm{n_o}}, \, \mathsf{Z_0}, \, \mathsf{U})$
\STATE $G.\texttt{AddVertex}(x_{\mathsf{f}}), \, G.\texttt{AddBP}((Q_{\mathsf{f}}, \, K_{\mathsf{f}}))$
\STATE $(Q_{\mathsf{new}}, \, K_{\mathsf{new}}) \leftarrow (Q_{\mathsf{f}}, \, K_{\mathsf{f}})$, $x_{\mathsf{new}} \leftarrow x_{\mathsf{f}}$
\WHILE{$x_{\mathsf{0}} \notin E_{\mathsf{new}} (\upvarepsilon)$}
\STATE $q_{\mathsf{rand}} \leftarrow \texttt{RandomConfiguration}(\bigcap_{\mathsf{i=1}}^{\mathrm{n_o}} \bar{\mathsf{a}}_\mathrm{i})$
\STATE $E_{\mathsf{near}} (\upvarepsilon) \leftarrow \texttt{NearestBP}(q_{\mathsf{rand}}, \, G, \, \upvarepsilon)$
\STATE $q_{\mathsf{new}} \leftarrow \texttt{NewEquilibrium}(q_{\mathsf{rand}}, \, E_{\mathsf{near}} (\upvarepsilon))$
\STATE $x_{\mathsf{new}} \leftarrow F(q_{\mathsf{new}})$
\STATE $(Q_{\mathsf{new}}, \, K_{\mathsf{new}}) \leftarrow \texttt{BP} (x_{\mathsf{new}}, \emptyset, \bar{\mathsf{a}}_{\mathsf{1}}, \cdots, \bar{\mathsf{a}}_{\mathrm{n_o}}, \mathsf{Z_0}, \mathsf{U})$
\STATE $G.\texttt{AddVertex}(x_{\mathsf{new}}), \, G.\texttt{AddBP}((Q_{\mathsf{new}}, \, K_{\mathsf{new}})),$
$G.\texttt{AddEdge}((x_{\mathsf{near}}, \, x_{\mathsf{new}}))$
\ENDWHILE
\STATE $(Q_{\mathsf{0}}, \, K_{\mathsf{0}}) \leftarrow \texttt{BP} (x_{\mathsf{a}_\mathsf{0}}, \, \mathsf{a}_{\mathsf{0}}, \, \bar{\mathsf{a}}_{\mathsf{1}}, \, \cdots, \, \bar{\mathsf{a}}_{\mathrm{n_o}}, \, \mathsf{Z_0}, \, \mathsf{U})$
\STATE $G.\texttt{AddVertex}(x_{\mathsf{0}}), \, G.\texttt{AddBP}((Q_{\mathsf{0}}, \, K_{\mathsf{0}})),$ $G.\texttt{AddEdge}((x_{\mathsf{new}}, \, x_{\mathsf{0}}))$
\end{algorithmic}
\end{algorithm}

In \cite{he2020bp}, a barrier pair rapidly-exploring random tree ($\mathsf{BP}$-$\mathsf{RRT}$) method is introduced which leverages rapidly-exploring random trees ($\mathsf{RRT}$) to combine a number of barrier pairs into a sequence that connects two polytopic regions in the reachable workspace. 
Compared $\mathsf{RRT}$, the $\mathsf{BP}$-$\mathsf{RRT}$ adds a barrier pair to each vertex in a graph providing additional robustness and safety guarantees for trajectory execution.

Algorithm~\ref{code:BP-RRT-1} shows the procedure for creating a $\mathsf{BP}$-$\mathsf{RRT}$ graph. 
Instead of applying a fixed incremental distance as $\mathsf{RRT}$ does in each iteration, a new robot configuration $q_{\mathsf{new}}$ is added to the graph by projecting a random configuration $q_{\mathsf{rand}}$ to the hyper-surface of the sub-level set $E_\mathsf{near} (\upvarepsilon)$ of the nearest barrier pair (for $0 < \upvarepsilon \leq 1$). 
Therefore, $q_{\mathsf{new}}$ is guaranteed to be inside the zero sub-level set of previously created barrier pairs. 
The algorithm terminates if there exists a new barrier pair $(Q_0, K_0)$ in the $\mathsf{BP}$-$\mathsf{RRT}$ graph whose zero sub-level set $E_0 (1)$ contains the entire region of $\mathsf{a_{0}}$.
Then, a sequence of barrier pairs that connects $\mathsf{a}_{\mathsf{0}}$ and $\mathsf{a}_{\mathsf{f}}$ can be extracted from the $\mathsf{BP}$-$\mathsf{RRT}$ graph. 

\subsection{Problem Statement}

In this paper, we consider a robot operating around multiple different polytopic regions defined in the workspace of its end-effector and a human operator that applies a norm-bound interaction force to the robot's end-effector intermittently. 
The goal of this human operator is to move the robot end-effector to the operator's target region.

{\bf Problem}:
During real-time human-robot shared autonomy operation, we aim to infer the operator's target region from a time series of intermittent human force measurements and create a sequence of barrier pairs such that the robot's end-effector can safely move to the target region without passing through all other regions.

\section{Barrier Pair Synthesis}

Similar to \cite{he2020bp}, our process of barrier pair synthesis starts by linearizing the robot's dynamics in \eqref{eq:ss-lg} and \eqref{eq:ss-jacobian} such that a norm-bound linear differential inclusion ($\mathsf{LDI}$) model can be formulated. 
Then, a $\mathsf{LMI}$ optimization problem can be created for synthesizing barrier pairs subject to predefined state space and input constraints. 
In particular, we formulate a $\mathsf{LMI}$ constraint in the barrier pair synthesis problem for enforcing the robot’s stability and the convergence of the barrier function under the influence of the norm-bound force input $w$ from the human. 

\subsection{Norm-Bound Linear Differential Inclusion Model}

Our barrier pair synthesis relies on solving an $\mathsf{LMI}$ optimization problem formulated based on a linear model of the robot dynamics. However, the linearized state space equations in \eqref{eq:ss-lg} and \eqref{eq:ss-jacobian} become inaccurate if the state $[ q ^ \top, \, \dot{q} ^ \top ] ^ \top$ deviates from the equilibrium.
In order to address this issue, we use a norm-bound $\mathsf{LDI}$ to represent the robot dynamics. First, we can express the norm-bound uncertainties of the linearized robot dynamical model in \eqref{eq:ss-lg} and \eqref{eq:ss-jacobian} as 
\begin{align}
- M ^ {-1} (q) \cdot C(q, \, \dot{q}) & \in \{ A_1 + A_2 \Delta A_3: \ \norm{\Delta} \leq 1 \} \label{eq:As} \\
M ^ {-1} (q) \cdot J ^ \top (q) & \in \{ \, B_1 ^ w + B_2 ^ w \Delta B_3 ^ w: \ \norm{\Delta} \leq 1 \} \label{eq:Bfs} \\
M ^ {-1} (q) & \in \{ B_1 ^ u + B_2 ^ u \Delta B_3 ^ u: \norm{\Delta} \leq 1 \} \label{eq:Bus} \\
J (q) & \in \{ \;\, J_1 + J_2 \Delta J_3 \;\,: \ \norm{\Delta} \leq 1 \} \label{eq:Js} 
\end{align}
for all state $[ q ^ \top, \, \dot{q} ^ \top ] ^ \top$ in the constrained state space $\mathsf{Z}$ around the equilibrium.
Then, a norm-bound $\mathsf{LDI}$ \cite{boyd1994linear} that is valid for all states in $\mathsf{Z}$ can be expressed as
\begin{align}
\begin{bmatrix}
\dot{\tilde{q}} \\
\ddot{\tilde{q}}
\end{bmatrix}
& =
\begin{bmatrix}
\mathbf{0} & I \\
\mathbf{0} & A_1 + A_2 \Delta A_3
\end{bmatrix}
\begin{bmatrix}
\tilde{q} \\
\dot{\tilde{q}}
\end{bmatrix}
+
\begin{bmatrix}
\mathbf{0} \\
B_1 ^ u + B_2 ^ u \Delta B_3 ^ u
\end{bmatrix}
u \notag \\
& +
\begin{bmatrix}
\mathbf{0} \\
B_1 ^ w + B_2 ^ w \Delta B_3 ^ w
\end{bmatrix}
w \label{eq:ss-lg-robust} \\
\tilde{x} \,
& =
\begin{bmatrix}
J_1 + J_2 \Delta J_3 & \mathbf{0}
\end{bmatrix}
\begin{bmatrix}
\tilde{q} \\
\dot{\tilde{q}}
\end{bmatrix}. \label{eq:ss-jacobian-robust}
\end{align}
We can formulate a norm-bound $\mathsf{LDI}$ by calculating $M ^ {-1} (q) \cdot C(q, \, \dot{q})$, $M ^ {-1} (q) \cdot J ^ \top (q, \, \dot{q})$, $M ^ {-1} (q)$ and $J (q)$ from a number of sample states in $\mathsf{Z}$ and using quadric inclusion programs \cite{thomas2019quadric} to fit an inclusion model.

The constrained state space region for the norm-bound $\mathsf{LDI}$ is defined as $\mathsf{Z} \triangleq \mathsf{Z_{safe}} \cap \mathsf{Z_{0}}$.
Based on the inequality constraints $|a_\mathrm{i}  \tilde{x}| < \bar{\mathit{a}}_\mathrm{i}$ associated with the undesirable regions $\mathsf{a_1, \, a_2, \, \cdots, \, a_{n_o}}$, a local convex state space region $\mathsf{Z_{safe}}$ can be defined as
\begin{equation} \label{eq:Qsafe}
\begin{aligned}
\mathsf{Z_{safe}} 
\triangleq 
\{ 
[ \tilde{q} ^ \top, \, \dot{\tilde{q}} ^ \top] ^ \top 
:
\abs{a_\mathrm{i}  (J_1 + J_2 \Delta J_3) \ \tilde{q}} < \bar{\mathit{a}}_\mathrm{i}, \\
\norm{\Delta} \leq 1, \
\mathrm{i = 1, \, \cdots, \, n_o}
\},
\end{aligned}
\end{equation}
where $a_\mathrm{i}$ for $\mathrm{i = 1, \, \cdots, \, n_o}$ are row vectors with $\mathrm{n_o}$ as the number of undesirable regions.
However, the norm-bound uncertainty in $\mathsf{Z_{safe}}$ can be too large for the barrier pair sub-problem to be solved. So we also need to consider an additional constrained state space $\mathsf{Z_0}$ defined as
\begin{equation} \label{eq:Q0}
\begin{aligned}
\mathsf{Z_0} 
\triangleq 
\{ 
\mathbf{[ \tilde{q} ^ \top, \, \dot{\tilde{q}} ^ \top] ^ \top} 
:
\abs{b_\mathrm{i}  (J_1 + J_2 \Delta J_3) \ \tilde{q}} < \bar{\mathit{x}}_\mathrm{i}, \
\abs{b_\mathrm{i}  \dot{\tilde{q}}} < \bar{\dot{\mathit{q}}}_\mathrm{i}, \\
\norm{\Delta} \leq 1, \
\mathrm{i = 1, \, \cdots, \, n}
\},
\end{aligned}
\end{equation}
where $b_\mathrm{i}$ for $\mathrm{i = 1, \, \cdots, \, n}$ are the standard basis (row) vectors of the $\mathrm{n}$-dimensional Euclidean space. 

Similar to \eqref{eq:Qsafe} and \eqref{eq:Q0}, a constrained input space region $\mathsf{U}$ and a constrained external input space region $\mathsf{W}$ can be formulated as
\begin{align}
\mathsf{U} 
& \triangleq 
\{ 
\mathbf{u} 
:
\abs{\mathbf{b_\mathrm{i}}  \mathbf{u}} < \bar{\mathit{u}}_\mathrm{i}, \
\mathrm{i = 1, \, \cdots, \, n}
\}. \label{eq:U} \\
\mathsf{W} 
& \triangleq 
\{ 
\mathbf{w} 
:
\norm{\mathbf{w}} < \bar{\mathit{w}}
\}. \label{eq:F}
\end{align}

\begin{figure}[!tbp]
\footnotesize
\centering
\def\svgwidth{0.5\textwidth}
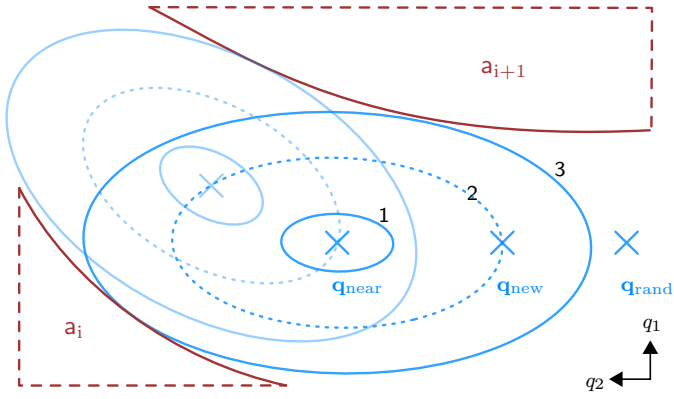
\caption{By projecting a random joint space position $q_{\mathsf{rand}}$ to the hyper-surface of $E_\mathsf{near}(\upepsilon_1)$ of the nearest barrier pair, a new equilibrium of $\mathsf{BP}$-$\mathsf{RRT}$ is created. The residue set $E_\mathsf{new}(\upepsilon_0)$ of the new barrier pair is designed to be strictly inside the zero sub-level set $E_\mathsf{near}(1)$ of the nearest barrier pair. Notice that even if the undesirable regions of the workspace are polytopic, their joint space projections are not guaranteed to be also polytopic. The numbers indicate $^1$ the residue set $E_\mathsf{near}(\upepsilon_0)$, $^2$ the hyper-surface of $E_\mathsf{near}(\upepsilon_1)$, and $^3$ the zero sub-level set $E_\mathsf{near}(1)$.}
\label{fig:fig-3}
\end{figure}

\subsection{Barrier Pair Synthesis Sub-Problems}

Based on the norm bound $\mathsf{LDI}$ model expressed in \eqref{eq:ss-lg-robust} and \eqref{eq:ss-jacobian-robust}, we can formulate the $\mathsf{LMI}$s for creating our barrier pair synthesis problem.
In order to ensure that the ellipsoidal sub-level set $E (1)$ of a barrier pair contains a desired polytopic region $\mathsf{a_{d}}$, we sample a number of points from all edges of $\mathsf{a_{d}}$ and let $E (1)$ contain the joint space projections of these Cartesian space samples using the following set of $\mathsf{LMI}$s
\begin{equation} \label{eq:x-inclusion}
\begin{bmatrix}
1 & \star \\
 R (x_\mathrm{i}) - q_e & S_1 Q  S_1 ^ \top 
\end{bmatrix} \succeq 0, \quad \forall \ \mathrm{i = 1, \dotsc, n_p} 
\end{equation}
where $\mathrm{n_p}$ is the number of sampled workspace points at the edge of $\mathsf{a}_{\mathsf{d}}$, $\mathsf{a}_{\mathsf{d}} = \mathsf{Co} \{x_1, \, \cdots, \, x_\mathrm{p} \}$, $R (\cdot)$ is an inverse kinematics operator and $S_1 \triangleq [ \mathbf{I}_\mathrm{n \times n}, \, \mathbf{0}_\mathrm{n \times n}]$.

Similar to \cite{he2020bp}, the constraints $\mathsf{Z_{safe}}$, $\mathsf{Z_{0}}$ and $\mathsf{U}$ in \eqref{eq:Qsafe}, \eqref{eq:Q0} and \eqref{eq:U} can be transformed into $\mathsf{LMI}$s
\begin{align}
\begin{bmatrix}
\bar{\mathit{a}}_\mathrm{i}^{2} Q & \star & \star & \star \\
\mathbf{0} & \upgamma_\mathrm{i} \mathbf{I} & \star & \star \\
a_\mathrm{i} J_1 S_1 Q & \upgamma_\mathrm{i} a_\mathrm{i} J_2 & 1 & \star \\
J_3 S_1 Q & \mathbf{0} & \vec{0} & \upgamma_\mathrm{i} \mathbf{I}
\end{bmatrix} & \succeq 0, \,
\forall \, \mathrm{i = 1, \dotsc, n_o}  \label{eq:x-exclusion} \\
\begin{bmatrix}
\bar{\mathit{x}}_\mathrm{i}^{2} Q & \star & \star & \star \\
\mathbf{0} & \upmu_\mathrm{i} \mathbf{I} & \star & \star \\
b_\mathrm{i} J_1 S_1 Q & \upmu_\mathrm{i} b_\mathrm{i} J_2 & 1 & \star \\
J_3 S_1 Q & \mathbf{0} & \vec{0} & \upmu_\mathrm{i} \mathbf{I}
\end{bmatrix} & \succeq 0, \,
\forall \, \mathrm{i = 1, \dotsc, n} \label{eq:x-limit} \\
\begin{bmatrix}
Q & \star \\
b_\mathrm{i} S_{2} Q & \bar{\dot{\mathit{q}}}_\mathrm{i} ^ 2
\end{bmatrix} & \succeq 0, \, \forall \, \mathrm{i = 1, \dotsc, n} \label{eq:qdot-limit} \\
\begin{bmatrix}
Q & \star \\
b_\mathrm{i} Y & \bar{\mathit{u}}_\mathrm{i} ^ 2
\end{bmatrix} & \succeq 0, \, \forall \, \mathrm{i = 1, \dotsc, n} \label{eq:u-limit}
\end{align}
where $\upgamma_\mathrm{i}$ for $\mathrm{i = 1, \dotsc, n_o}$ and $\upmu_\mathrm{i}$ for $\mathrm{i = 1, \dotsc, n}$ are positive real scalar variables, $S_2 \triangleq [ \mathbf{0}_\mathrm{n \times n}, \, \mathbf{I}_\mathrm{n \times n}]$ and $Y \triangleq K Q$.
We use the $\mathsf{S}$-procedure \cite{ma2006lmi} to transfer the the workspace position constraints defined in $\mathsf{Z_{safe}}$ and $\mathsf{Z_{0}}$ into $\mathsf{LMI}$s in \eqref{eq:x-exclusion} and \eqref{eq:x-limit}. $\mathsf{LMI}$s in \eqref{eq:qdot-limit} represent the joint velocity constraints in $\mathsf{Z_{0}}$. $\mathsf{LMI}$s in \eqref{eq:u-limit} represent the input constraints in $\mathsf{U}$.


By the following proposition, the robot's Lyapunov stability can be enforced under the impact of norm-bound human input $w$.

\begin{proposition}
For a robot starting from a state in the zero sub-level set $E(1)$ of the barrier pair $(B, K)$, the robot state converges to a residue set $E (\upvarepsilon_0)$ with an exponential convergence rate no less than $\frac{\upalpha}{2}$ if
\begin{equation} \label{eq:stability}
\begin{aligned}
\begin{bmatrix}
X_{11} & \star \\
X_{21} & X_{22} \\
\end{bmatrix} \preceq 0,
\end{aligned}
\end{equation}
where
{\small
}
\begin{align}
& X_{11}
=
\text{\small$
\begin{bmatrix}
\bar{A}_1 Q + Q \bar{A}_1 ^ \top + \bar{B}_1 ^ u Y + Y ^ \top \bar{B}_1 ^ {u \top} + \upalpha Q & \bar{B}_1 ^ w \\
\bar{B}_1 ^ {w \top} & - \upalpha \frac{\upvarepsilon_0^2}{\bar{\mathit{w}} ^ 2} \mathbf{I}
\end{bmatrix}
$} \notag
\\
& + 
\text{\small$
\begin{bmatrix}
\bar{A}_2 & \bar{B}_2 ^ u & \bar{B}_2 ^ w \\
\mathbf{0} & \mathbf{0} & \mathbf{0} \\
\end{bmatrix} 
\begin{bmatrix}
\upmu_x \mathbf{I} & \mathbf{0} & \mathbf{0} \\
\mathbf{0} & \upmu_u \mathbf{I} & \mathbf{0} \\
\mathbf{0} & \mathbf{0} & \upmu_w \mathbf{I} \\
\end{bmatrix}
\begin{bmatrix}
\bar{A}_2 & \bar{B}_2 ^ u & \bar{B}_2 ^ w \\
\mathbf{0} & \mathbf{0} & \mathbf{0} \\
\end{bmatrix} 
^ \top \label{eq:X11}
$}
\\
& X_{21} 
=
\begin{bmatrix}
\bar{A}_3 Q & \mathbf{0} \\
B_3 ^ u Y & \mathbf{0} \\
\mathbf{0} & B_3 ^ w \\
\end{bmatrix}
\\
& X_{22}
=
\begin{bmatrix}
- \upmu_x \mathbf{I} & \mathbf{0} & \mathbf{0} \\
\mathbf{0} & - \upmu_u \mathbf{I} & \mathbf{0} \\
\mathbf{0} & \mathbf{0} & - \upmu_w \mathbf{I} \\
\end{bmatrix} \\
& \bar{A}_1 = S_1 ^ \top S_2 + S_2 ^ \top A_1 S_2, \, \bar{A}_2 = S_2 ^ \top A_2, \, \bar{A}_3 = A_3 S_2, \label{eq:A-bar} \\
& \bar{B}_1 ^ u = S_2 ^ \top B_1 ^ u, \ \bar{B}_2 ^ u = S_2 ^ \top B_2 ^ u, \label{eq:Bu-bar} \\
& \bar{B}_1 ^ w = S_2 ^ \top B_1 ^ w, \ \bar{B}_2 ^ w = S_2 ^ \top B_2 ^ w, \label{eq:Bw-bar}
\end{align}
and $\upmu_\mathrm{x}$, $\upmu_\mathrm{u}$ and $\upmu_\mathrm{w}$ are positive real scalar variables.
\end{proposition}
\begin{proof}
See Appendix~\ref{proof:stability}.
\end{proof}

Finally, the volume of the ellipsoid $E (1)$ is maximized through the cost function of the log of the determinant of $Q$ \cite{boyd1994linear}. 
A barrier pair synthesis sub-problem $(B, \, k) = \texttt{BP} (x_e, \, \mathsf{a}_{\mathsf{d}}, \, \bar{\mathsf{a}}_{\mathsf{1}}, \, \cdots, \, \bar{\mathsf{a}}_{\mathrm{n_o}}, \, \mathsf{Z_0}, \, \mathsf{U}, \, \mathsf{W})$ can be expressed as
\begin{equation} \label{optimization}
\begin{aligned}
& \underset{Q, \, Y}{\mathsf{maximize}}
& & \mathsf{log}(\mathsf{det} (Q)) \\
& \mathsf{subject \ to} & & Q \succ 0, \\
&&& \eqref{eq:x-inclusion}, \, \eqref{eq:x-exclusion}, \, \eqref{eq:x-limit}, \, \eqref{eq:qdot-limit}, \, \eqref{eq:u-limit}, \, \eqref{eq:stability} \\
\end{aligned}
\end{equation}
for finding a sub-level set $E (1)$ that contains the desired region $\mathsf{a}_{\mathsf{d}}$ and excludes the undesirable regions $\mathsf{a_1, \, a_2, \, \cdots, \, a_{n_o}}$.

\section{Human-Robot Shared Autonomy}

Based on the barrier pairs synthesized from the $\mathsf{LMI}$ problem formulated in \eqref{optimization}, we propose a two-layer control framework for addressing the interaction problem defined in Section~II.C. In the first layer, we employ the history of human input measurements to infer what the human wants the robot to do and define the robot's safety constraints based on it.
In the second layer, we use the $\mathsf{BP}$-$\mathsf{RRT}$ method \cite{he2020bp} to generate sequences of barrier pairs that safely move the robot's end-effector to different desired regions. Based on the human intention inference performed in the first layer, the robot can execute these barrier pair sequences accordingly and help to accomplish the human's objective. 

\begin{algorithm}[t]
\caption{$G \leftarrow \texttt{BPRRT}$\small$(\mathsf{a}_{\mathsf{0}}, \mathsf{a}_{\mathsf{f}}, \bar{\mathsf{a}}_{\mathsf{1}}, \cdots, \bar{\mathsf{a}}_{\mathrm{n_o}}, \mathsf{Z_0}, \mathsf{U}, \mathsf{W}, \upvarepsilon_0, \upvarepsilon_1)$} \label{code:BP-RRT-2}
\begin{algorithmic} [1]
\REQUIRE Initial region $\mathsf{a}_{\mathsf{0}}$, goal region $\mathsf{a}_{\mathsf{f}}$, constraints associated with undesirable regions $\bar{\mathsf{a}}_{\mathsf{1}}, \, \cdots, \, \bar{\mathsf{a}}_{\mathrm{n_o}}$, state space constraint $\mathsf{Z_0}$, robot input constraint $\mathsf{U}$, human input constraint $\mathsf{W}$, scalar $\upvarepsilon_0$ $(0 < \upvarepsilon_0 \leq 1)$, scalar $\upvarepsilon_1$ $(0 < \upvarepsilon_1 \leq 1)$
\ENSURE $\mathsf{BP}$-$\mathsf{RRT}$ graph $G$
\STATE $(Q_{\mathsf{f}}, \, K_{\mathsf{f}}) \leftarrow \texttt{BP} (x_{\mathsf{a}_\mathsf{f}}, \, \mathsf{a}_{\mathsf{f}}, \, \bar{\mathsf{a}}_{\mathsf{1}}, \, \cdots, \, \bar{\mathsf{a}}_{\mathrm{n_o}}, \, \mathsf{Z_0}, \, \mathsf{U}, \mathsf{W}, \upvarepsilon_0)$
\STATE $G.\texttt{AddVertex}(x_{\mathsf{f}}), \, G.\texttt{AddBP}((Q_{\mathsf{f}}, \, K_{\mathsf{f}}))$
\STATE $(Q_{\mathsf{new}}, \, K_{\mathsf{new}}) \leftarrow (Q_{\mathsf{f}}, \, K_{\mathsf{f}})$, $x_{\mathsf{new}} \leftarrow x_{\mathsf{f}}$
\WHILE{$x_{\mathsf{0}} \notin E_{\mathsf{new}} (\upvarepsilon_1)$}
\STATE $q_{\mathsf{rand}} \leftarrow \texttt{RandomConfiguration}(\bigcap_{\mathsf{i=1}}^{\mathrm{n_o}} \bar{\mathsf{a}}_\mathrm{i})$
\STATE $E_{\mathsf{near}} (\upvarepsilon_1) \leftarrow \texttt{NearestBP}(q_{\mathsf{rand}}, \, G, \, \upvarepsilon_1)$
\STATE $q_{\mathsf{att}} \leftarrow \texttt{NewEquilibrium}(q_{\mathsf{rand}}, \, E_{\mathsf{near}} (\upvarepsilon_1))$
\STATE $x_{\mathsf{att}} \leftarrow F(q_{\mathsf{att}})$
\STATE $(Q_{\mathsf{att}}, \, K_{\mathsf{att}}) \leftarrow \texttt{BP}$\small$(x_{\mathsf{att}}, \emptyset, \bar{\mathsf{a}}_{\mathsf{1}}, \cdots, \bar{\mathsf{a}}_{\mathrm{n_o}}, \mathsf{Z_0}, \mathsf{U}, \mathsf{W}, \upvarepsilon_0)$
\vspace{5pt}
\STATE $\upvarepsilon_2 \leftarrow \sqrt{[ q_\mathsf{near} ^ \top - q_\mathsf{att} ^ \top, \, \vec{0} ^ {\, \top}] ^ \top \ Q_\mathsf{att} ^ {-1} \ [ q_\mathsf{near} ^ \top - q_\mathsf{att} ^ \top, \, \vec{0} ^ {\, \top}]}$
\vspace{5pt}
\IF{$Q_\mathsf{att} \preceq \frac{(1 - \upvarepsilon_1) ^ 2}{\upvarepsilon_0 ^ 2} \cdot Q_\mathsf{near}$ $\mathbf{and}$ $Q_\mathsf{near} \preceq \frac{(1 - \upvarepsilon_2) ^ 2}{\upvarepsilon_0 ^ 2} \cdot Q_\mathsf{att}$}
\STATE $(Q_{\mathsf{new}}, \, K_{\mathsf{new}}) \leftarrow (Q_{\mathsf{att}}, \, K_{\mathsf{att}})$, $x_{\mathsf{new}} \leftarrow x_{\mathsf{att}}$
\STATE $G.\texttt{AddVertex}(x_{\mathsf{new}}), \, G.\texttt{AddBP}((Q_{\mathsf{new}}, \, K_{\mathsf{new}})),$
$G.\texttt{AddEdge}((x_{\mathsf{near}}, \, x_{\mathsf{new}}))$
\ENDIF
\ENDWHILE
\STATE $(Q_{\mathsf{0}}, \, K_{\mathsf{0}}) \leftarrow \texttt{BP} (x_{\mathsf{a}_\mathsf{0}}, \, \mathsf{a}_{\mathsf{0}}, \, \bar{\mathsf{a}}_{\mathsf{1}}, \, \cdots, \, \bar{\mathsf{a}}_{\mathrm{n_o}}, \, \mathsf{Z_0}, \, \mathsf{U}, \mathsf{W}, \upvarepsilon_0)$
\STATE $G.\texttt{AddVertex}(x_{\mathsf{0}}), \, G.\texttt{AddBP}((Q_{\mathsf{0}}, \, K_{\mathsf{0}})),$ $G.\texttt{AddEdge}((x_{\mathsf{new}}, \, x_{\mathsf{0}}))$
\end{algorithmic}
\end{algorithm}

\subsection{Barrier Pair Sampling}

Although Algorithm~\ref{code:BP-RRT-1} provides us the steps for creating a $\mathsf{BP}$-$\mathsf{RRT}$ graph, it cannot be used directly to solve the interaction problem because of the additional human input $w$. Therefore, we need to formulate a new barrier pair sampling algorithm for creating a barrier pair sequence that moves the robot's end-effector safely to a human desired region $\mathsf{a}_d$ under the human input intervention. 

We extend our $\mathsf{BP}$-$\mathsf{RRT}$ method to a robot under a human force input $w$, as outlined in Algorithm~\ref{code:BP-RRT-2}. 
The algorithm initializes the $\mathsf{BP}$-$\mathsf{RRT}$ graph by creating a barrier pair at $\mathsf{a}_{\mathsf{f}}$ in line~1-3, expands it by sampling new barrier pairs in line~4-15, and completes it by creating a barrier pair at $\mathsf{a}_{\mathsf{0}}$ in line~16-17. 

Different from Algorithm~\ref{code:BP-RRT-1}, Algorithm~\ref{code:BP-RRT-2} considers two scalar inputs $\upvarepsilon_0$ and $\upvarepsilon_1$ (Fig.~\ref{fig:fig-3}).
The first scalar input $\upvarepsilon_0$, previously introduced in \eqref{eq:X11}, defines the residue set $E(\upvarepsilon_0)$ of a barrier pair. 
Similar to the scalar input $\upvarepsilon$ in Algorithm~\ref{code:BP-RRT-1}, the second scalar input $\upvarepsilon_1$ defines a hyper-surface of sub-level set $E_\mathsf{near} (\upvarepsilon_1)$ of the nearest barrier pair found in line~6 such that a new equilibrium $q_\mathsf{att}$ can be obtained by projecting a random configuration $q_\mathsf{rand}$ sampled in line 5 to this hyper-surface. 

In order to enforce the robot's safe transition between two barrier pairs of an edge in the graph, the residue set $E_\mathsf{att}(\upvarepsilon_0)$ of the newly sampled barrier pair created in line 9 needs to be completely inside the zero sub-level set $E_\mathsf{near} (1)$ of the nearest barrier pair found in line 6. We can check this safety requirement through the condition stated in the following proposition. 

\begin{proposition}
Suppose $(Q_1, K_1)$ and $(Q_2, K_2)$ represent two barrier pairs forming an edge in a $\mathsf{BP}$-$\mathsf{RRT}$ graph. Let $z_1 \triangleq [ q_1 ^ \top, \, \vec{0} ^ {\, \top}] ^ \top$ and $z_2 \triangleq [ q_2 ^ \top, \, \vec{0} ^ {\, \top}] ^ \top$ be the equilibrium points of $(Q_1, K_1)$ and $(Q_2, K_2)$ located at the hyper-surface of $E_2 (\upvarepsilon_2)$ and $E_1 (\upvarepsilon_1)$, respectively. Let $E_1 (\upvarepsilon_0)$ and $E_2 (\upvarepsilon_0)$ be the residue sets of $(Q_1, K_1)$ and $(Q_2, K_2)$. The robot can safely transit between the zero sub-level sets $E_1 (1)$ and $E_2 (1)$ of these two barrier pairs if $Q_1 \preceq \frac{(1 - \upvarepsilon_2) ^ 2}{\upvarepsilon_0 ^ 2} \cdot Q_2$ and $Q_2 \preceq \frac{(1 - \upvarepsilon_1) ^ 2}{\upvarepsilon_0 ^ 2} \cdot Q_1$.
\end{proposition}
\begin{proof}
See Appendix~\ref{proof:proposition}.
\end{proof}

Line 11 in Algorithm~\ref{code:BP-RRT-2} checks the condition in Proposition~2. Notice that this condition guarantees the safe transition between $(Q_\mathsf{near}, K_\mathsf{near})$ and $(Q_\mathsf{att}, K_\mathsf{att})$ in both directions. Therefore, although the graph is initialized from $\mathsf{a_{f}}$ and expanded toward $\mathsf{a_{0}}$, we can finally extract a sequence of barrier pairs which plan safe robot trajectories from $\mathsf{a_{0}}$ to $\mathsf{a_{f}}$ and from $\mathsf{a_{f}}$ to $\mathsf{a_{0}}$.

\subsection{Human Intention Inference}

Based on the concept of Boltzmann rationality~\cite{baker2007goal, morgenstern1953theory}, we propose our human intent inference method for interpreting the human input $w$ in the shared autonomy.
Boltzmann rationality formalizes intent according to a variable that quantifies the value of the human's actions.
In particular, it states that a rational human takes an action with probability proportional to the exponentiated value of that human action. Therefore, an action with higher value is more probable to be chosen by the human.

In the setting of robotic manipulation considered in this paper, we define the value of the human's action based on how well the human force $w$ aligns with the direction toward the human's goal. 
Let $\mathsf{a}$ denote the human's goal, $x_t$ denote the position of the robot's end-effector at time ${t}$, and $w_t$ denote the human force exerted at time ${t}$. 
Recall that $x_\mathsf{a_\mathrm{i}}$ is the center of a polytopic region $\mathsf{a_\mathrm{i}}$ in the workspace. 
We define the likelihood function of exerting the force $w_t$ conditioned on the true human's goal $\mathsf{a}$ as
\begin{equation}\label{eq:likelihood}
\mathsf{p} (w_t \mid \mathsf{a}=\mathsf{a_\mathrm{i}}) = {\upbeta_0 \cdot \mathrm{exp}(\upbeta_1 \cdot \langle w_t, \, x_\mathsf{a_\mathrm{i}} - x_t \rangle)},
\end{equation}
where $\upbeta_0 > 0$, $\upbeta_1 > 0$, and the value of the human's action is captured by the inner product of $w_t$ and the direction toward the target region $x_\mathsf{a_\mathrm{i}} - x_t$. The value of this inner product indicates how well the exerted force is correlated with the direction toward the target region. 
$\upbeta_0$ is a partition function defined as 
\begin{equation}
\upbeta_0 ^ {-1} = \int\limits_{w \in \mathsf{W}} {\mathrm{exp}(\upbeta_1 \cdot \langle w_t, \, x_\mathsf{a_\mathrm{i}} - x_t \rangle)} \, \mathrm{d}w,
\end{equation}
where $\mathsf{W}$ is the domain of feasible human force input defined in \eqref{eq:F}. $\upbeta_1$ is the rationality parameter representing the degree of human's rationality.

Now, using the likelihood function, we can compute and update the robot's belief over the human's intended target region. Let us define the robot's belief as
\begin{equation}
    \mathsf{b_\mathrm{t}} (\mathsf{a_\mathrm{i}}) = \mathsf{p} (\mathsf{a}=\mathsf{a_\mathrm{i}} \mid w_0, \, \ldots, \, w_t)
\end{equation}
which denotes the probability of the target region being $\mathsf{a_\mathrm{i}}$ given the history of human's inputs. Initially, the system starts with a uniform belief, i.e., $\mathsf{b_\mathrm{0}} \sim \texttt{unif}\{\mathrm{1},\mathrm{n_o}\}$. Then, we can update the belief using the Bayes' theorem 
\begin{equation}
\mathsf{b_\mathrm{t}} (\mathsf{a_\mathrm{i}}) = \frac{\hfill \mathsf{b_\mathrm{t-1}} (\mathsf{a_\mathrm{i}}) \cdot \mathsf{p} (w_t \mid \mathsf{a}=\mathsf{a_\mathrm{i}})}{\sum_{\mathsf{j = 1}}^{\mathrm{n_o}} \mathsf{b_\mathrm{t-1}} (\mathsf{a_\mathrm{j}}) \cdot \mathsf{p} (w_t \mid \mathsf{a}=\mathsf{a_\mathrm{j}})},
\end{equation}
where $\mathsf{p} (w_t \mid \mathsf{a}=\mathsf{a_\mathrm{i}})$ is computed according to \eqref{eq:likelihood}. 

At time $t$, the robot's belief is used to select the sequence of barrier pairs that carry out the task of safely going to the estimated target region $\hat{\mathsf{a}}(\mathrm{t})$ calculated as
\begin{equation}
    \hat{\mathsf{a}}(\mathrm{t}) = \argmax_{i \in \{\mathrm{1},\mathrm{2},\ldots,\mathrm{n_o}\}} \mathsf{b_\mathrm{t}} (\mathsf{a_\mathrm{i}}),
\end{equation}
which has the highest probability of being the human's intended goal.

\section{Example}

In this section, we demonstrate the proposed control framework through a simulation of a 2-link manipulator robot with an equal length of $0.75 \, \mathrm{m}$ for each link, a mass of $2.5 \, \mathrm{kg}$ located at the distal end of each link, and a torque limit of $25 \, \mathrm{N \cdot m}$ for each joint. 
Fig.~\ref{fig:fig-8} shows the polytopic regions in the workspace of the robot end effector, where $\mathsf{a}_1$, $\mathsf{a}_2$, and $\mathsf{a}_3$ represent the desired task regions, $\mathsf{a}_4$, $\mathsf{a}_5$, and $\mathsf{a}_6$ represent obstacle regions, and $\mathsf{a}_7$ represents the region where the robot's base is located.


The manipulator robot starts from an end-effector position in $\mathsf{a}_1$.
A human operator decides whether to apply a $1$ $\mathrm{N}$ force to the end-effector during the simulation. 
The human operator chooses the direction of the $1$ $\mathrm{N}$ force from 8 different possible directions through a keyboard.

\subsection{Barrier Pair Synthesis}

We use Algorithm~\ref{code:BP-RRT-2} to build barrier pair sequences which connect between $\mathsf{a_1}$, $\mathsf{a_2}$, and $\mathsf{a_3}$ (Fig.~\ref{fig:fig-1}.a-c). 
Barrier pairs $\mathsf{c}_1$, $\mathsf{c}_2$, and $\mathsf{c}_3$ are in the middle of the sequences from $\mathsf{a}_2$ to $\mathsf{a}_3$, from $\mathsf{a}_3$ to $\mathsf{a}_1$, and from $\mathsf{a}_1$ to $\mathsf{a}_2$, respectively. 
Sometimes, the inference of the human target region may change and result in the barrier pair sequence currently executed by the robot to be invalid. 
Therefore, we also use Algorithm~\ref{code:BP-RRT-2} to connect between $\mathsf{c}_1$, $\mathsf{c}_2$, and $\mathsf{c}_3$ (Fig.~\ref{fig:fig-1}.d-f) such that the robot can freely switch between the correct barrier pair sequences without going through any undesired regions. Fig.~\ref{fig:fig-6} shows the transitions between $\mathsf{a_1}$, $\mathsf{a_2}$, $\mathsf{a_3}$, $\mathsf{c}_1$, $\mathsf{c}_2$, and $\mathsf{c}_3$.

\subsection{Simulation}

The video of this simulation is available at \url{https://youtu.be/xTprU0jMT8w}.

The simulated manipulator robot uses the measurement of the $1$ $\mathrm{N}$ force to infer the human operator's desired goal region.
As the video shows, the initial force input from the human operator is sometimes ambiguous because it can point to multiple potential goal regions.
However, the proposed human intention inference method is able to successfully recover the intended goal fast enough such that the manipulator robot does not move its end-effector to an incorrect goal region.  

\begin{figure}[!tbp]
\centering
\def\svgwidth{0.49\textwidth}
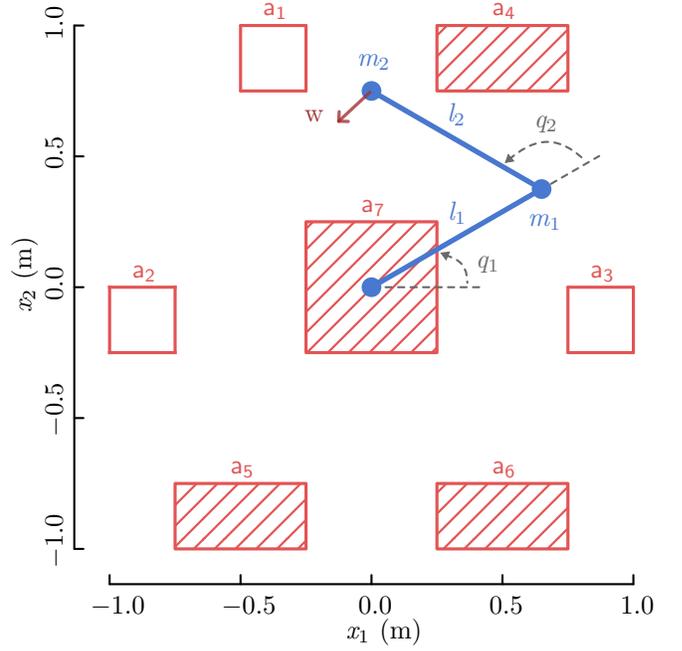
\caption{A 2-link manipulator robot moves its end effector in a workspace with different polyoptic regions. $\mathsf{a}_1$, $\mathsf{a}_2$, and $\mathsf{a}_3$ represent the desired task regions, $\mathsf{a}_4$, $\mathsf{a}_5$, and $\mathsf{a}_6$ represent obstacle regions, and $\mathsf{a}_7$ represents the region where the robot's base is located.}
\label{fig:fig-8}
\end{figure}

\begin{figure}[!tbp]
\centering
\def\svgwidth{0.49\textwidth}
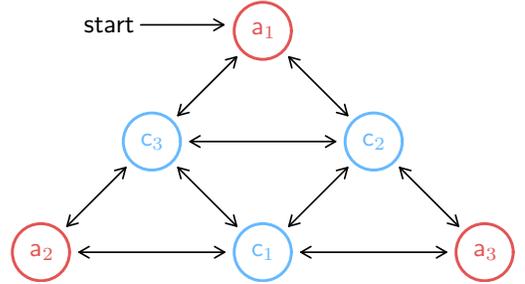
\caption{A finite state machine indicates the possible transitions between $\mathsf{a_1}$, $\mathsf{a_2}$, $\mathsf{a_3}$, $\mathsf{c}_1$, $\mathsf{c}_2$, and $\mathsf{c}_3$.}
\label{fig:fig-6}
\end{figure}

\begin{figure*}
\footnotesize
\centering
\def\svgwidth{1.0\textwidth}
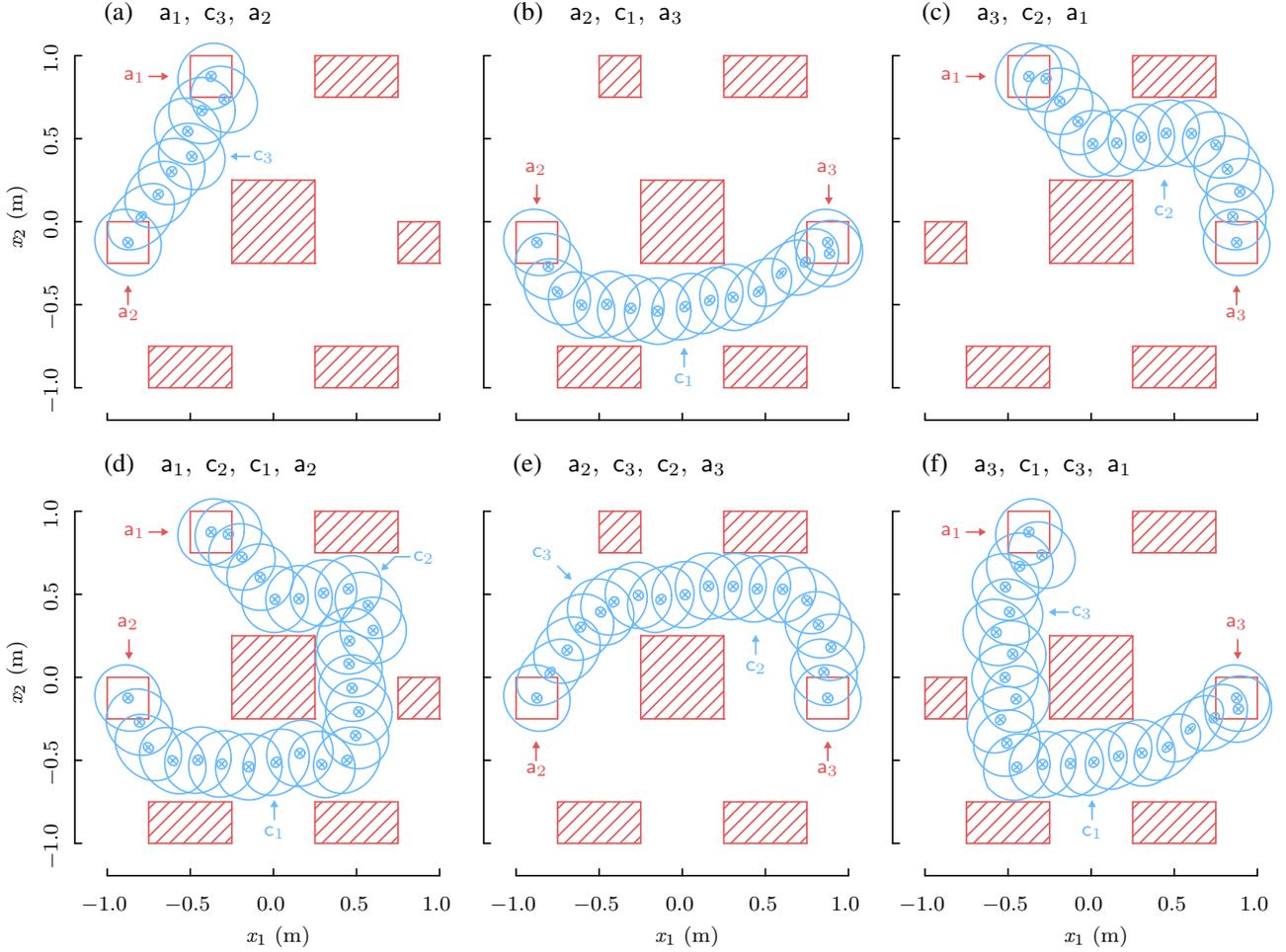
\caption{Barrier pair sequences connect $\mathsf{a_{init}}$ and $\mathsf{a_{goal}}$ and avoid passing through the undesirable polytopic regions. The bigger and the smaller ellipsoids indicate the zero sub-level sets and the residue sets of the barrier pairs projected onto the $\mathit{x}_1$-$\mathit{x}_2$ space. The arrows point to $\mathsf{a_1}$, $\mathsf{a_2}$, $\mathsf{a_3}$, $\mathsf{c}_1$, $\mathsf{c}_2$, $\mathsf{c}_3$.}
\label{fig:fig-1}
\end{figure*}

\section{Discussion}

In this research, our two-layer control framework uses a Bayesian inference method to infer the human operator's objective while uses the state feedback controllers in the barrier pairs to guarantee the safe operation. Compared to force-based control strategies \cite{buerger2007complementary, he2019modeling} that consider the robot as a strict human follower, the state feedback controllers allow the robot to correct the potentially unsafe robot trajectories caused by the human force intervention. While a lot of other safety control methods such as the quadratic programming methods \cite{ames2016control, nguyen2016optimal, nguyen2016exponential} also enforce the satisfaction of state-space constraints, maintaining safety under robot input saturation is difficult. To guarantee the safe operation under input saturation, we incorporate the $\mathsf{LMI}$s in \eqref{eq:u-limit} into the synthesis of our barrier pairs.  

Before running Algorithm~\ref{code:BP-RRT-2}, we need to decide the values of $\upepsilon_0$ and $\upepsilon_1$. 
In general, a larger value of $\upepsilon_0$ makes the $\mathsf{LMI}$ problem in \eqref{optimization} more solvable. However, we also need to have $\upepsilon_0 < 1 - \upepsilon_1$ such that the condition $Q_\mathsf{att} \preceq \frac{(1 - \upvarepsilon_1) ^ 2}{\upvarepsilon_0 ^ 2} \cdot Q_\mathsf{near}$ in line 11 of Algorithm~\ref{code:BP-RRT-2} can be fulfilled. 
How to select the value of $\upepsilon_1$ is also a complicated problem.
On the one hand, a smaller value of $\upepsilon_1$ allows us to select a larger value of $\upepsilon_0$.
On the other hand, a smaller value of $\upepsilon_1$ also means we need to synthesize more barrier pairs for each sequence.
In the provided example, we choose $\upepsilon_0 = 0.15$ and $\upepsilon_1 = 0.80$ based on all these considerations.

In \cite{he2020bp}, the transition between two barrier pairs in a $\mathsf{BP}$-$\mathsf{RRT}$ graph is unidirectional, so it is hard for the robot to return to its former location or switch between different barrier pair sequences. In this paper, we resolve this issue by enforcing the condition in Proposition~2. This condition is checked after the barrier pair synthesis because $Q_\mathsf{near} \preceq \frac{(1 - \upvarepsilon_2) ^ 2}{\upvarepsilon_0 ^ 2} \cdot Q_\mathsf{att}$ in line 11 of Algorithm~\ref{code:BP-RRT-2} is a non-convex $\mathsf{LMI}$ constraint and cannot be included in the convex optimization problem defined in \eqref{optimization}. However, this also means $Q_\mathsf{att} \preceq \frac{(1 - \upvarepsilon_1) ^ 2}{\upvarepsilon_0 ^ 2} \cdot Q_\mathsf{near}$ in line 11 of Algorithm~\ref{code:BP-RRT-2} is a convex $\mathsf{LMI}$ constraint, which can be potentially added to our barrier pair synthesis problem.

In the provided example, we create three additional barrier pair sequences (Fig.~\ref{fig:fig-1}.d-f) to connect the midway barrier pairs $\mathsf{c}_1$, $\mathsf{c}_2$, and $\mathsf{c}_3$ such that the robot's end-effector can safely switch between the three original barrier pair sequences (Fig.~\ref{fig:fig-1}.a-c).  
But this strategy also creates more barrier pair sequences than necessary. 
The number of barrier pair sequences will increase dramatically if we consider more polytopic regions in Fig.~\ref{fig:fig-8} as the possible human desired regions.
A potential approach to resolve this issue is considering all possible human desired regions in one barrier pair sampling algorithm and creating a roadmap \cite{kavraki1996probabilistic} instead of just one sequence of barrier pairs.


\appendix
\subsection{Proof of Proposition 1} \label{proof:stability}

Let us define $z \triangleq [ \tilde{q} ^ \top, \, \dot{\tilde{q}} ^ \top] ^ \top$. Based on the matrices defined in \eqref{eq:A-bar}, \eqref{eq:Bu-bar} and \eqref{eq:Bw-bar}, \eqref{eq:ss-lg-robust} can be expressed as
\begin{equation}
\dot{z} = \bar{A}_1 z + \bar{B}_1^u u + \bar{B}_1^w w + \bar{A}_2 p_z + \bar{B}_2^u p_u + \bar{B}_2^w p_w 
\end{equation}
where
\begin{align}
p_z & = \Delta q_z, \quad q_z = \bar{A}_3 z  , \label{eq:px} \\
p_u & = \Delta q_u, \quad q_u =     {B}_3^u u, \label{eq:pu} \\
p_w & = \Delta q_w, \quad q_w =     {B}_3^w w. \label{eq:pw} 
\end{align}
For barrier function $B = z ^ \top P z - 1$ with controller $u = K z$, the time derivative of $B$ is
\begin{equation}
\dot{B} 
= 
\text
{\tiny
$
\begin{bmatrix}
z \\
w \\
p_z \\
p_u \\
p_w \\
\end{bmatrix}
^ \top 
\begin{bmatrix}
(\bar{A}_1 + \bar{B}_1 ^ u K) ^ {  \top} P + P (\bar{A}_1 + \bar{B}_1 ^ u K) & \star & \star & \star & \star  \\
\bar{B}_1 ^ {w \top} P & \mathbf{0} &      \star &      \star &      \star \\
\bar{A}_2 ^ {  \top} P & \mathbf{0} & \mathbf{0} &      \star &      \star \\
\bar{B}_2 ^ {u \top} P & \mathbf{0} & \mathbf{0} & \mathbf{0} &      \star \\
\bar{B}_2 ^ {w \top} P & \mathbf{0} & \mathbf{0} & \mathbf{0} & \mathbf{0} \\
\end{bmatrix}
\begin{bmatrix}
z \\
w \\
p_z \\
p_u \\
p_w \\
\end{bmatrix}
$
}
.
\end{equation}
In addition, we have
\begin{align}
w ^ \top   w & \leq \bar{\mathit{w}} ^ 2, \label{eq:ww}  \\
z ^ \top P z & \leq \upvarepsilon_0 ^ 2,           \label{eq:zPz}
\end{align}
for the norm-bound human input $w$ and the residue set $\{ z \mid B \leq \upvarepsilon_0 ^ 2 - 1 \}$ of the barrier function.

Using the $\mathsf{S}$-procedure, we can combine \eqref{eq:px}, \eqref{eq:pu}, \eqref{eq:pw}, \eqref{eq:ww}, \eqref{eq:zPz}, and $\dot{B} \leq 0$ into
\begin{equation} \label{eq:appendix-1}
\begin{bmatrix}
\tilde{X}_{11} & \star          & \star \\
\tilde{X}_{21} & \tilde{X}_{22} & \star \\
\mathbf{0}     & \mathbf{0}     & - \upalpha \upvarepsilon_0 ^ 2 + \upalpha_w \bar{\mathit{w}} ^ 2
\end{bmatrix}
\preceq
0.
\end{equation}
where
\begin{align}
\tilde{X}_{11}
& =
\text{\small$
\begin{bmatrix}
(\bar{A}_1 + \bar{B}_1 ^ u K) ^ \top P + P (\bar{A}_1 + \bar{B}_1 ^ u K) + \upalpha P & P \bar{B}_1 ^ w \\
\bar{B}_1 ^ {w \top} P & - \upalpha_w \mathbf{I}
\end{bmatrix}
$} \notag
\\
& + 
\text{\small$
\begin{bmatrix}
\bar{A}_3 & \mathbf{0} \\
B_3 ^ u K & \mathbf{0} \\
\mathbf{0} & B_3 ^ w \\
\end{bmatrix}
^ \top 
\begin{bmatrix}
\uplambda_x \mathbf{I} & \mathbf{0} & \mathbf{0} \\
\mathbf{0} & \uplambda_u \mathbf{I} & \mathbf{0} \\
\mathbf{0} & \mathbf{0} & \uplambda_w \mathbf{I} \\
\end{bmatrix}
\begin{bmatrix}
\bar{A}_3 & \mathbf{0} \\
B_3 ^ u K & \mathbf{0} \\
\mathbf{0} & B_3 ^ w \\
\end{bmatrix}
$}
\\
\tilde{X}_{21} 
& =
\begin{bmatrix}
\bar{A}_2 ^ {  \top} P & \mathbf{0} \\
\bar{B}_2 ^ {u \top} P & \mathbf{0} \\
\bar{B}_2 ^ {w \top} P & \mathbf{0} \\
\end{bmatrix} 
\\
\tilde{X}_{22}
& =
\begin{bmatrix}
- \uplambda_x \mathbf{I} & \mathbf{0} & \mathbf{0} \\
\mathbf{0} & - \uplambda_u \mathbf{I} & \mathbf{0} \\
\mathbf{0} & \mathbf{0} & - \uplambda_w \mathbf{I} \\
\end{bmatrix}
\end{align}
Without loss of generality, we can let $\upalpha_w = \upalpha \frac{\upvarepsilon_0 ^ 2}{\bar{\mathit{w}} ^ 2}$ such that \eqref{eq:appendix-1} becomes
\begin{equation} \label{eq:appendix-2}
\begin{bmatrix}
\tilde{X}_{11} & \star          \\
\tilde{X}_{21} & \tilde{X}_{22} 
\end{bmatrix}
\preceq
0
\end{equation}
which is equivalent to \eqref{eq:stability} for $\upmu_x = \frac{1}{\uplambda_x}$, $\upmu_u = \frac{1}{\uplambda_u}$, $\upmu_w = \frac{1}{\uplambda_w}$, and $Q = P ^ {-1}$.

\balance

\subsection{Proof of Proposition 2} \label{proof:proposition}

Let us define a vector norm function
\begin{align} \label{eq:vector-norm}
\norm{\star}_{Q_1} & \triangleq \sqrt{\star ^ \top Q_1 ^ {-1} \star}
\end{align}
based on the quadratic part in the barrier function of $(Q_1, K_1)$. Because equilibrium $z_2$ of $(Q_2, K_2)$ is on hyper-surface of $E_1 (\upvarepsilon_1)$, we have
\begin{align} \label{eq:proposition-2-1}
\norm{z_2 - z_1}_{Q_1} = \upvarepsilon_1.
\end{align}
Suppose $z_1 ^ \prime$ is a point on the hyper-surface of $E_1 (1)$, we have
\begin{align} \label{eq:proposition-2-2}
\norm{z_1 ^ \prime - z_2}_{Q_1} + \norm{z_2 - z_1}_{Q_1} \geq \norm{z_1 ^ \prime - z_1}_{Q_1} = 1
\end{align}
because of the triangle inequality of $\norm{\star}_{Q_1}$.
Based on \eqref{eq:proposition-2-1} and \eqref{eq:proposition-2-2}, we have 
\begin{align} \label{eq:proposition-2-3}
\norm{z_1 ^ \prime - z_2}_{Q_1} \geq 1 - \upvarepsilon_1
\end{align}
which is equivalent to 
\begin{align} \label{eq:proposition-2-4}
\{ z \mid (z - z_2) ^ \top Q_1 ^ {-1} (z - z_2) \leq (1 - \upvarepsilon_1) ^ 2 \} \subseteq E_1 (1).
\end{align}
If $Q_1 \preceq \frac{(1 - \upvarepsilon_2) ^ 2}{\upvarepsilon_0 ^ 2} \cdot Q_2$, we have 
\begin{align} \label{eq:proposition-2-5}
E_2 (\upvarepsilon_0) \subseteq \{ z \mid (z - z_2) ^ \top Q_1 ^ {-1} (z - z_2) \leq (1 - \upvarepsilon_1) ^ 2 \}.
\end{align}
Combining \eqref{eq:proposition-2-4} and \eqref{eq:proposition-2-5}, we have $E_2 (\upvarepsilon_0) \subseteq E_1 (1)$. If the robot starts from any states in $E_2 (1)$, it safely converges to a subset in $E_1 (1)$ using barrier pair $(Q_2, K_2)$.

Similarly, we have $E_1 (\upvarepsilon_0) \subseteq E_2 (1)$ if $Q_2 \preceq \frac{(1 - \upvarepsilon_1) ^ 2}{\upvarepsilon_0 ^ 2} \cdot Q_1$.


\bibliographystyle{IEEEtran}
\bibliography{main}

\balance

\end{document}